\documentclass[letterpaper, 10 pt, conference]{ieeeconf}

\IEEEoverridecommandlockouts                              %

\title{\LARGE \bf
db-CBS: Discontinuity-Bounded Conflict-Based Search for Multi-Robot Kinodynamic Motion Planning
}

\author{Akmaral Moldagalieva, Joaquim Ortiz-Haro, Marc Toussaint, and Wolfgang Hönig%
\thanks{All authors are with Technical University of Berlin, Berlin, Germany
        {\tt\footnotesize moldagalieva@tu-berlin.de}.}%
\thanks{Code: \url{https://github.com/IMRCLab/db-CBS}}
\thanks{Video: \url{https://youtu.be/1mglNQOmOLE}}
}

\usepackage[bookmarks=true]{hyperref}
\usepackage{amssymb}
\usepackage{amsmath}
\usepackage{multirow}
\usepackage[detect-weight=true, detect-family=true,detect-inline-weight=math]{siunitx}

\newtheorem{theorem}{Theorem}
\newtheorem{definition}{Definition}

\usepackage[detect-weight=true, detect-family=true,detect-inline-weight=math]{siunitx}

\usepackage[ruled,vlined,linesnumbered]{algorithm2e}
\SetInd{0.1em}{0.4em}
\SetAlFnt{\footnotesize\sf}
\SetKwRepeat{Do}{do}{while}
\SetKw{KwStep}{step}
\SetKwInput{KwData}{Input}
\SetKwInput{KwResult}{Result}
\SetKwBlock{RepeatForever}{repeat}{}
\SetKw{Continue}{continue}
\SetKwComment{Comment}{$\triangleright$\ }{}
\SetCommentSty{itshape}

\usepackage[capitalise]{cleveref} %
\crefname{equation}{}{} %

\usepackage[natbib=true,
    style=ieee,
    citestyle=numeric-comp,
    backend=bibtex,
    maxbibnames=10,
    maxcitenames=1,
    mincitenames=1,
    doi=false,
    url=false,
    giveninits=true]{biblatex}
\renewcommand*{\bibfont}{\footnotesize}

\usepackage{tikz}
\usetikzlibrary{fit,calc}
\newcommand*{\tikzmk}[1]{\tikz[remember picture,overlay,] \node (#1) {};\ignorespaces}

\newcommand{\marklineSix}[1]{\tikz[remember picture,overlay]{\node[yshift=2pt,xshift=#1,fill=yellow!100,opacity=.25,fit={(A)($(A)+(0.95\linewidth,-5.3\baselineskip)$)},rounded corners=4pt] {};}\ignorespaces}

\newcommand{\marklineFour}[1]{\tikz[remember picture,overlay]{\node[yshift=2pt,xshift=#1,fill=yellow!100,opacity=.25,fit={(A)($(A)+(0.95\linewidth,-3.3\baselineskip)$)},rounded corners=4pt] {};}\ignorespaces}

\newcommand{\marklineThree}[1]{\tikz[remember picture,overlay]{\node[yshift=2pt,xshift=#1,fill=yellow!100,opacity=.25,fit={(A)($(A)+(0.95\linewidth,-2.3\baselineskip)$)},rounded corners=4pt] {};}\ignorespaces}

\newcommand{\marklineTwo}[1]{\tikz[remember picture,overlay]{\node[yshift=2pt,xshift=#1,fill=yellow!100,opacity=.25,fit={(A)($(A)+(0.95\linewidth,-1.3\baselineskip)$)},rounded corners=4pt] {};}\ignorespaces}

\newcommand{\marklineOne}[1]{\tikz[remember picture,overlay]{\node[yshift=2pt,xshift=#1,fill=yellow!100,opacity=.25,fit={(A)($(A)+(0.95\linewidth,-0.3\baselineskip)$)},rounded corners=4pt] {};}\ignorespaces}

\usepackage{balance}

\newcommand{\vx}{\mathbf{x}}    %
\newcommand{\vu}{\mathbf{u}}    %
\newcommand{\vg}{\mathbf{g}}    %

\newcommand{\seqX}{\mathbf{X}}    %
\newcommand{\seqU}{\mathbf{U}}    %

\newcommand{\sX}{\mathcal{X}}   %
\newcommand{\sU}{\mathcal{U}}   %
\newcommand{\sW}{\mathcal{W}}   %
\newcommand{\sM}{\mathcal{M}}   %
\newcommand{\sR}{\mathcal{R}}  %
\newcommand{\sB}{\mathcal{B}} %
\newcommand{\sC}{\mathcal{C}}   %
\newcommand{\sO}{\mathcal{O}}   %
\newcommand{\vf}{\mathbf{f}}    %
\DeclareMathOperator{\step}{step}

\begin{document}

\maketitle
\thispagestyle{empty}
\pagestyle{empty}

\begin{abstract}
This paper presents a multi-robot kinodynamic motion planner that enables a team of robots with different dynamics, actuation limits, and shapes to reach their goals in challenging environments.
We solve this problem by combining Conflict-Based Search (CBS), a multi-agent path finding method, and discontinuity-bounded A*, a single-robot kinodynamic motion planner. Our method, db-CBS, operates in three levels.
Initially, we compute trajectories for individual robots using a graph search that allows bounded discontinuities between precomputed motion primitives. The second level identifies inter-robot collisions and resolves them by imposing constraints on the first level. 
The third and final level uses the resulting solution with discontinuities as an initial guess for a joint space trajectory optimization. The procedure is repeated with a reduced discontinuity bound.
Our approach is anytime, probabilistically complete, asymptotically optimal, and finds near-optimal solutions quickly. Experimental results with robot dynamics such as unicycle, double integrator, and car with trailer in different settings show that our method is capable of solving challenging tasks with a higher success rate and lower cost than the existing state-of-the-art.

\end{abstract}

\section{Introduction}
Multi-robot systems have a wide-range of real-world applications including delivery, collaborative transportation, and search-and-rescue. One of the essential requirements to enhance the autonomy of a team of robots in these settings is being able to reach the goal quickly while avoiding collisions with obstacles and other robots. Moreover, the planned motions are required to respect the robots' dynamics which impose constraints on their velocity or acceleration. 
Considering the complexity of Multi-Robot Motion Planning (MRMP), the majority of existing solutions either make simplified assumptions like ignoring the robot dynamics or actuation limits, produce highly suboptimal solutions, or do not scale well with the number of robots.

In this paper, we present an efficient, probabilistically-complete and asymptotically optimal motion planner for a heterogeneous team of robots which takes into account constraints imposed by robot dynamics. Our method leverages the Multi-Agent Path Finding (MAPF) optimal solver Conflict-Based Search (CBS), the single-robot kinodynamic motion planner discontinuity-bounded A* (db-A*), and nonlinear trajectory optimization. 

We are motivated by the success of solving multi-agent path finding problems efficiently with bounded suboptimality guarantees using the family of Conflict-Based Search (CBS) algorithms.
Our algorithm, called Discontinuity-Bounded Conflict-Based Search (db-CBS), performs a three-level search. On the low-level the single-robot motion is planned separately for each robot relying on pre-computed motion primitives obeying the robot's dynamics. The computed trajectories may result in inter-robot collisions, which are resolved one-by-one in the second-level search.
Different from CBS, we allow the trajectories to be dynamically infeasible up to a given discontinuity bound during these two searches.
Then, we use the output trajectories as an initial guess for a joint space trajectory optimization on the third level. 
This is repeated with a lower discontinuity bound in case of failure or if a lower-cost solution is desired.
In this way, we combine the advantages of informed discrete search and optimization to find a near-optimal solution quickly.

The main contribution of this work is a new kinodynamic motion planner for multi-robot systems. We compare our method with other kinodynamic motion planning methods on the same problem instances with the identical objective of computing time-optimal trajectories. Finally, we demonstrate that our planned motions can be safely executed on a team of flying robots. 

\cref{fig:main} shows db-CBS in application with a team of eight heterogeneous robots in simulation and with four robots in real-world settings.

\begin{figure}
    \centering
    \includegraphics[height=3.5cm]{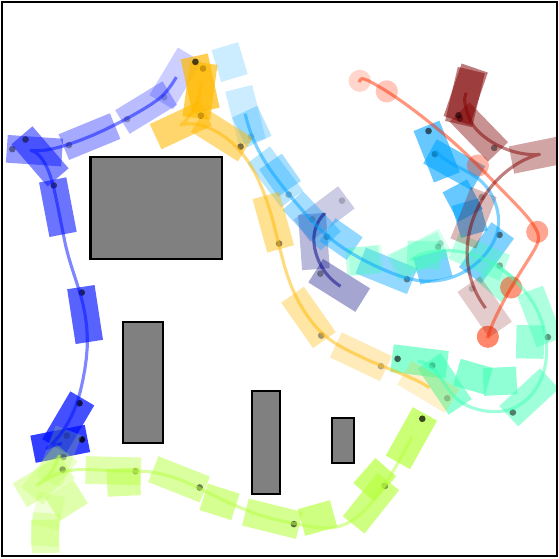} 
    \includegraphics[height=3.5cm]{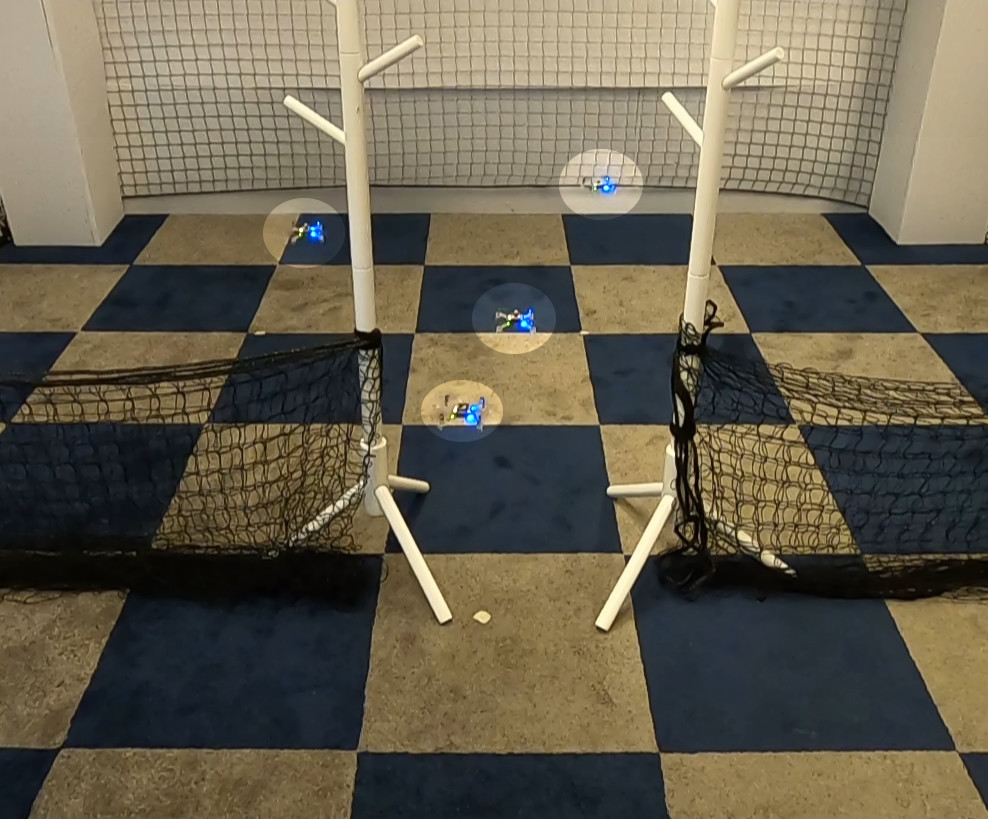}
    \caption{ Db-CBS solving multi-robot kinodynamic motion planning.
    Left: problem instance with a team of 8 heterogeneous robots: unicycle $1^{\text{st}}$ order (rectangle), double integrator (circle), car with trailer (rectangle with box in the end). Right: Real-world experiment challenging each pair of flying robots to pass the narrow window swapping their positions in order to reach their assigned goals.}
    \label{fig:main}
\end{figure}

\section{Related Work}

Multi-Agent Path Finding (\textbf{MAPF}) assumes a discrete state space represented as a graph. A robot can move from one vertex along an edge to an adjacent neighbor in one step; robots cannot occupy the same vertex or traverse the same edge at the same timestep.
For collision checking, vertices are often placed in a grid, but any-angle extensions exist~\cite{yakovlev2017}.
Finding an optimal MAPF solution is NP-hard~\cite{yu_2013}. Current algorithms plan paths for each robot individually and then resolve conflicts by fixing velocities~\cite{cui2012}, assigning priorities to robots~\cite{cap2013} or adding path constraints as in CBS~\cite{cbs}.
Although MAPF has practical applications in real-world problems, it ignores robot dynamics and can result in kinodynamically infeasible solutions.
The planned paths may be followed using a controller and plan-execution framework~\cite{wolfgang2016}.
However, this is suboptimal, may result in collisions, and does not generalize to all robot types.

\textbf{Kinodynamic} motion planning remains challenging, even for single robots. Search-based methods with motion primitives have been adapted to a variety of robotic systems, including high-dimensional systems~\cite{dharmadhikari2020}. Those motion primitives are on a state lattice, thus they are able to connect states properly and can be designed manually~\cite{cohen2010}. Afterwards, any variant of discrete path planning can be used without modifications. Sampling-based planners expand a tree of collision-free and dynamically feasible motions~\cite{kavraki1996,steven2001,sst}. Even though solutions have probabilistic completeness guarantees, they are suboptimal and require post-processing to smooth the trajectory. Optimization-based planners rely on an initial guess and optimize locally, for example by using sequential convex programming (SCP)~\cite{schulman2014}. 
Hybrid methods can combine search and sampling~\cite{sakcak2019}, search and optimization~\cite{natarajan2021} or sampling and optimization~\cite{sanjiban2016}. In this work, we extend db-A*~\cite{dbastar} - a hybrid method connecting ideas from search, sampling, and optimization.

For multi-robot kinodynamic motion planning, single-robot planners applied to the joint space can be used, but do not scale well beyond a few robots.
Extensions of CBS can tackle the Multi-Robot Motion Planning (\textbf{MRMP}) problems. For example, discretized workspace grid cells can be connected via predefined motion primitives~\cite{liron2019}. However, computing primitives remains challenging for higher-order dynamics.
Another option is to use model predictive control~\cite{cbsmpc} or a sampling-based method~\cite{kcbs} as the low-level planner.
The latter, called K-CBS, uses any sampling-based planner and merges the search space of two robots if the number of conflicts between the robots exceeds a threshold. In the worst case, all robots may be merged but probabilistic completeness is achieved.
In Safe Multi-Agent Motion Planning with Nonlinear Dynamics and Bounded Disturbances (S2M2)~\cite{s2m2},  Mixed-Integer Linear Programs (MILP) and control theory are combined by computing piecewise linear paths that a controller should be able to track within a safe region. These safe regions can be large, which makes the method incomplete.

Db-CBS, also leverages CBS, but we use a different single-robot planner and trajectory optimization, resulting in an efficient algorithm with completeness guarantees.

\section{Problem Definition}

We consider $N$ robots and denote the state of the $i^{\text{th}}$ robot with $\vx^{(i)} \in \sX^{(i)} \subset \mathbb R^{d_{x^{(i)}}}$, which is actuated by controlling actions $\vu^{(i)} \in \sU^{(i)} \subset \mathbb R^{d_{u^{(i)}}}$. 
The workspace the robots operate in is given as $\sW \subseteq \mathbb R^{d_w}$ ($d_w\in\{2,3\}$). The collision-free space is $\sW_{\mathrm{free}} \subseteq \sW$. 

We assume that each robot $i \in \{1,\ldots,N\}$ has dynamics:
\begin{equation}
    \label{eq:dynamics}
    \dot \vx^{(i)} = \vf^{(i)}(\vx^{(i)}, \vu^{(i)}).
\end{equation}
The Jacobian of $\vf^{(i)}$ with respect to $\vx^{(i)}$ and $\vu^{(i)}$ are assumed to be available in order to use gradient-based optimization. 

With zero-order hold discretization, \cref{eq:dynamics} can be framed as
\begin{equation}
    \label{eq:dynamics_discrete}
    \vx_{k+1}^{(i)} \approx \step(\vx_k^{(i)}, \vu_k^{(i)}) \equiv \vx_k^{(i)} + \vf(\vx^{(i)}_k, \vu_k^{(i)})\Delta t,
\end{equation}
 where $\Delta t$ is sufficiently small to ensure that the Euler integration holds.

We use $\seqX^{(i)} = \langle \vx_0^{(i)}, \vx_1^{(i)}, \ldots, \vx_{K^{(i)}}^{(i)} \rangle$ as a sequence of states of the $i^{\text{th}}$ robot sampled at times $0, \Delta t, \dots, K^{(i)} \Delta t$ and $\seqU^{(i)} = \langle \vu_0^{(i)}, \vu_1^{(i)}, \ldots, \vu_{K^{(i)}-1}^{(i)} \rangle$ as a sequence of actions applied to the $i^{\text{th}}$ robot for times $[0,\Delta t), [\Delta t, 2\Delta t), \ldots, [(K^{(i)}-1)\Delta t, K^{(i)}\Delta t)$.
Our goal is to move the team of $N$ robots from their start states $\vx_s^{(i)} \in \sX^{(i)}$ to their goal states  $\vx_f^{(i)} \in \sX^{(i)}$ as fast as possible without collisions. This can be formulated as the following optimization problem:
\begin{align}
    &\min_{\{\seqX^{(i)\}},\{\seqU^{(i)}\},\{K^{(i)}\}} \sum_{i=1}^{N} K^{(i)} \label{eq:opt_mrs}
    \\
    &\text{\noindent s.t.}\begin{cases}
    \vx_{k+1}^{(i)} = \step(\vx_k^{(i)}, \vu_k^{(i)}) & \forall i\; \forall k, \\
    \vu^{(i)}_k \in \sU^{(i)}, \,\,\,\, \vx^{(i)}_k \in \sX^{(i)} & \forall i\; \forall k, \\
    \sB^{(i)}(\vx_k^{(i)}) \in \sW_{\mathrm{free}}  & \forall i \; \forall k,  \\
    \sB^{(i)}(\vx_k^{(i)}) \cap  \sB^{(j)}(\vx_k^{(j)}) = \emptyset  & \forall i \neq j\; \forall k, \\
    \vx^{(i)}_0 = \vx_s^{(i)}, \,\,\,\, \vx^{(i)}_{K^{(i)}} = \vx_f^{(i)} & \forall i, \\
    \end{cases} \nonumber
\end{align}
where $\sB^{(i)}: \sX^{(i)} \to 2^\sW$ is a function that maps the configuration of the $i^{\text{th}}$ robot to a collision shape. 
The objective is to minimize the sum of the arrival time of all robots.

\section{Approach}\label{sec:approach}

\subsection{Background}
We first describe db-A* and CBS in more detail, since our work generalizes them into a multi-robot case operating in continuous space and time. 

\textbf{db-A*} is a generalization of A* for kinodynamic motion planning of a single robot that searches a graph of implicitly defined \textit{motion primitives}, i.e., precomputed motions respecting the robot's dynamics~\cite{dbastar}. 
 
A single motion primitive can be defined as a tuple $ \langle \seqX, \seqU, K \rangle$, consisting of state sequences  $\seqX = \langle \vx_0,..., \vx_K \rangle$ and control sequences $\seqU = \langle \vu_0, ..., \vu_{K-1} \rangle$, which obey the dynamics $\vx_{k+1} = \step(\vx_k, \vu_k)$. These motion primitives are used as graph edges to connect states, representing graph nodes, with user-configurable discontinuity $\delta$.

As in A*, db-A* explores nodes based on $f(\vx)=g(\vx)+h(\vx)$, where $g(\vx)$ is the cost-to-come. 
The node with the lowest $f$-value is expanded by applying valid motion primitives.
The output of db-A* is a $\delta$-discontinuity-bounded solution as defined below.
 
\begin{definition}
	\label{definition:discontinuityBounded}
	Sequences $\seqX = \langle \vx_0, \ldots, \vx_K \rangle$, $\seqU = \langle \vu_0, \ldots, \vu_{K-1}\rangle$ are \emph{$\delta$-discontinuity-bounded} solutions iff the following conditions hold:
    \begin{align}
        d(\vx_{k+1}, \step(\vx_k, \vu_k)) \leq \delta \, \forall k,\\
        \vu_k \in \sU,
        \sB(\vx_k) \in \sW_{\mathrm{free}}  \, \forall k,                       \nonumber\\
        d(\vx_0, \vx_s) \leq \delta, \,\,\,\,\,
        d(\vx_K, \vx_f) \leq \delta,\nonumber
    \end{align}
\end{definition}
where $d$ is a \emph{metric} $d: \sX \times \sX \to \mathbb R$, which measures the distance between two states.

To repair discontinuities in the trajectory, db-A* combines discrete search with gradient-based trajectory optimization in an iterative manner, decreasing $\delta$ with each iteration.

\textbf{CBS} is an optimal MAPF solver which works in two levels. 
The high-level search finds conflicts between single-robot paths, while at the low-level individual paths are updated to fulfill time-dependent constraints using a single-robot planner.
Each high-level node contains a set of constraints and, for each robot, a path which obeys the given constraints. The cost of the high-level node is equal to the sum of all single-robot path costs. When a high-level node $P$ is expanded, it is considered as the goal node, if its individual robot paths have no conflicts. If the node $P$ has no conflicts, then its set of paths is returned as the solution. Otherwise, two $P_1$ and $P_2$ child nodes are added. Both child nodes inherit constraints of $P$ and have an additional constraint to resolve the last conflict. For example, if the collision happened at vertex $s$ at time $t$ between robots $i$ and $j$, then the node $P_1$ has an additional constraint $\langle i,s,t \rangle$, which prohibits the robot $i$ to occupy vertex $s$ at time $t$. Likewise, the constraint $\langle j,s,t \rangle$ is added to $P_2$. Afterwards, given the set of constraints for each node, the low-level planner is executed for the affected robot, only.

\subsection{db-CBS}

Our kinodynamic planning approach, db-CBS, adapts CBS to multi-robot kinodynamic motion planning. 
The general structure of the db-CBS consists of the following steps: i) a single-robot path with discontinuous jumps is computed for each robot using db-A*, ii) collisions between individual paths are resolved one-by-one, 
iii) discontinuous motions are repaired into smooth and feasible trajectories with optimization.
These steps are repeated until a solution is found or the desired solution quality is achieved.
Although the two-level framework of CBS remains unchanged in our approach, the definition of conflicts and constraints need to be changed for the continuous domain. 
We define a conflict as $C = \langle i, j, \vx^{(i)}_k, \vx^{(j)}_k, k \rangle$ for a collision between robot $i$ with state $\vx^{(i)}_k$ and robot $j$ with state $\vx^{(j)}_k$ identified at time $k$. The resulting constraint for robot $i$ is $\langle i, \vx^{(i)}_k, k \rangle$, which prevents it to be within a distance of $\delta$ to state $\vx^{(i)}_k$ at time $k$. Similarly, the constraint for robot $j$ is $\langle j, \vx^{(j)}_k, k \rangle$.
The notion of a discontinuity $\delta$ helps us here to define the constraint as an actual volume (around a point), which is crucial for the efficiency and completeness guarantees.

The high-level search of db-CBS is shown in \cref{alg:dbcbs}. 
Its major changes compared to CBS are highlighted.
Db-CBS finds a solution in an iterative manner and decreases the value of the discontinuity allowed in the final path with each iteration (\cref{alg:dbcbs:overview:delta}). 
In each iteration, the root node of the constraint tree is initialized by running the low-level planner for each individual robot separately with the given set of motion primitives, updated $\delta$ and no constraints (\cref{alg:dbcbs:ll_start}). 
Each single-robot motion consists of a concatenation of motion primitives that avoid robot-obstacle collisions as well as the specified constraints.
If single-robot motions are found successfully, they are validated for inter-robot collision (\cref{alg::dbcbs:earliest_conflict}) by checking sequentially at each timestep if a collision between any two robots occurred. 
The checking terminates when the earliest conflict is detected between a pair of robots.
This conflict is resolved as in the original CBS by creating constraints and computing trajectories with the low-level planner for each constrained robot (\cref{alg::dbcbs:constraint_loop_begin}-\cref{alg::dbcbs:constraint_loop_end}). 
Db-CBS uses an extension of db-A* that can directly consider dynamic constraints efficiently for the low-level search.
If the motions have no conflicts, they are considered valid solution and used as an initial guess for the trajectory optimization (\cref{alg:overview:dbcbs:opt}). The trajectory optimization is performed in the joint configuration space with: 
\begin{align}
      \label{eq:mrs_opt}
   &   \min_{\{\mathbf{X}^{(i)}\}, \{\mathbf{U}^{(i)} \}, \Delta t  }  \sum_{i=1}^{N} \sum_{k=0}^{K-1}  ~ J(\vx_k^{(i)},\vu_k^{(i)}) \\ 
      &\text{\noindent s.t.}\begin{cases} 
        \begin{bmatrix}
            \vx_{k+1}^{(1)} \\
            \vx_{k+1}^{(2)} \\
              \vdots\\
            \vx_{k+1}^{(N)}
        \end{bmatrix} = 
        \begin{bmatrix}
            \vx_{k}^{(1)} \\
            \vx_{k}^{(2)} \\
              \vdots\\
            \vx_{k}^{(N)}
        \end{bmatrix} + 
         \begin{bmatrix}
            \vf(\vx_{k}^{(1)},\vu_k^{(1)}) \\
            \vf(\vx_{k}^{(2)},\vu_k^{(2)}) \\
              \vdots\\
            \vf(\vx_{k}^{(N)},\vu_k^{(N)}) 
        \end{bmatrix} \Delta t\;\; \forall k, \;
        \\ %
        \begin{bmatrix}
            \vu_{k}^{(1)},
            \vu_{k}^{(2)},
            \hdots,
            \vu_{k}^{(N)}
        \end{bmatrix}^\top \in \bar\sU\;\; \forall k, 
        \\
        \begin{bmatrix}
            \vx_{k}^{(1)},
            \vx_{k}^{(2)},
            \hdots,
            \vx_{k}^{(N)}
        \end{bmatrix}^\top \in \bar\sX\;\; \forall k, 
        \\
        \vg\left(\begin{bmatrix}\vx_{k}^{(1)}, \vx_{k}^{(2)}, \ldots, \vx_{k}^{(N)}\end{bmatrix}^\top\right) \ge \textbf{0} \;\; \forall k,\\
        \vx^{(i)}_0 = \vx_s^{(i)}, \,\,\,\,\, \vx^{(i)}_{K^{(i)}} = \vx_f^{(i)} \;\; \forall i.
        \end{cases} \nonumber
\end{align} 
Here, the cost is $J(\vx_k^{(i)},\vu_k^{(i)}) = \Delta t + \beta ||\vu_k^{(i)}||^2$ with a regularization $\beta$, and $\bar\sU = \times_{i} \sU^{(i)}$, $\bar\sX = \times_{i} \sX^{(i)}$ are the joint action space and joint state space, respectively.
The optimization variables are state and control actions of all robots, and the length of the time interval $\Delta t$. The signed-distance function $\vg()$ performs robot-robot and robot-obstacle collision avoidance for each state. 
Since our ultimate objective is to minimize the sum of arrival times, we add the goal constraints at different timesteps $K^{(i)}$ for each robot based on the discontinuity-bounded initial guess.
The trajectory optimization problem \cref{eq:mrs_opt} can be solved with nonlinear, constrained optimization. In our work we use differential dynamic programming (DDP), adding the goal and collision constraints with a squared penalty method.

\subsection{db-A* For Dynamic Obstacles}
In CBS and db-CBS, constraints arise directly from conflicts and represent states to be avoided during the low-level search, which requires planning with dynamic obstacles.
In CBS, the most common approach is to use A* in a space-time search space or SIPP~\cite{DBLP:conf/aips/HuHGSS22}.
Inspired by SIPP, we extend db-A* in order to solve single-robot path planning which is consistent with constraints. Our extended db-A* is given in \cref{alg:dbAstarMRS}.
The notation $\vx\oplus m$ indicates that the motion $m$ is applied to state $\vx$. 

Recall that a constraint $\langle i, \vx_c, k\rangle$ enforces $d(\vx_c, \vx_k^{(i)}) > \delta$.
We handle all constraints during the expansion of neighboring states, where we only include motions that are at least $\delta$ away from any constrained state (\cref{alg:dbAstarMRS:constraint_loop_begin} - \cref{alg:dbAstarMRS:check_collision}).
In addition, avoiding dynamic obstacles might require to reach a state with a slower motion.
Thus, we keep a list of safe arrival time, parent node, and motion to reach this state from the parent (\cref{alg:dbAstarMRS:Oinit}).
When a potentially better path is found, we keep previous solution motions rather than removing them (\cref{alg:dbAstarMRS:newRewiring}).
Conceptually, this is similar to SIPP, except that we do not store safe arrival intervals but potential arrival time points, since not all robots are able to wait. 
Empirically, storing the arrival times compared to always creating new nodes is significantly faster in our experiments.

\begin{algorithm}[t]
    \caption{db-CBS: Discontinuity-Bounded Conflict-Based Search (High-Level Search)}
    \label{alg:dbcbs}
    \DontPrintSemicolon
    \SetVlineSkip{2pt}

    \SetKwFunction{AddPrimitives}{AddPrimitives}
    \SetKwFunction{ExtractPrimitives}{ExtractPrimitives}
    \SetKwFunction{ComputeDelta}{ComputeDelta}
    \SetKwFunction{DiscontinuityBoundedAstar}{db-A*}
    \SetKwFunction{Optimization}{Optimization}
    \SetKwFunction{DisableMotions}{DisableMotions}
    \SetKwFunction{Report}{Report}
    
    \SetKwFunction{GetIndividualPaths}{GetIndividualPaths}
    \SetKwFunction{GetSolutionCost}{GetSolutionCost}
    \SetKwFunction{GetEarliestConflict}{GetEarliestConflict}
    \SetKwFunction{ValidatePath}{ValidatePath}
    \SetKwFunction{NewNode}{Node}
    \SetKwFunction{NearestNeighborInit}{NearestNeighborInit}
    \SetKwFunction{DisableMotions}{DisableMotions}
    \SetKwFunction{PriorityQueuePop}{PriorityQueuePop}
    \SetKwFunction{GetConstraints}{GetConstraints}
    \SetKwFunction{PriorityQueueInsert}{PriorityQueueInsert}
    \tcc{Main changes compared to CBS are highlighted.}
    \KwData{$ \{\vx_s^{(i)}\}, \{\vx_f^{(i)}\}, \sW_{\mathrm{free}}, N $}
    \KwResult{$\{\seqX^{(i)}\}, \{\seqU^{(i)} \}$}
    
    $\sM = \emptyset$ \Comment*{Initial set of motion primitives}
    \tikzmk{A}\For{$n=1,2,\ldots$ \label{alg:dbCBS:expand2}}{
        \For{$r \in N $ \label{alg:dbCBS:expand1}}{
            $\sM^{(r)} \leftarrow \sM^{(r)} \cup \AddPrimitives(r)$\label{alg:dbcbs:overview:get_M} \Comment*{Add motion primitives for each robot dynamics}
        }\marklineSix{-10pt} 
        $\delta_n \leftarrow \ComputeDelta()$ \Comment*{Update discontinuity bound} \label{alg:dbcbs:overview:delta}\;
        $\mathcal S \leftarrow \NewNode (solution: \emptyset, constraints: \emptyset)$   \Comment*{Root node} \label{alg:dbcbs:start_node}
        \tikzmk{A}\ForEach{$robot\: i \:\in\: \sR $}{
            $\mathcal S.solution[i] \leftarrow$\DiscontinuityBoundedAstar{$\vx_s^{(i)}, \vx_f^{(i)}, \sW_{\mathrm{free}}, \sM^{(i)}, \delta, None$} \label{alg:overview:dbcbs:dbAstar} \Comment*{Single-robot planner with no constraints} \label{alg:dbcbs:ll_start}
         }\marklineFour{-15pt} 
        $\mathcal S.cost \leftarrow \GetSolutionCost(\mathcal S.solution)$  \Comment*{Update the node cost} \label{alg:dbcbs:start_cost}
        $\sO \leftarrow \{\mathcal S\} $ \Comment*{Initialize open priority queue} \label{alg::dbcbs:open_init}
        \While{$|\sO| > 0$}{
            $P \leftarrow \PriorityQueuePop(\sO)$ \Comment*{Lowest solution cost}
            $C \leftarrow \GetEarliestConflict(P.solution)$  \Comment*{Check for collisions between individual motions} \label{alg::dbcbs:earliest_conflict}
            \eIf{$C = \emptyset$}{ 
               \tikzmk{A}$\{\seqX^{(i)} \}, \{\seqU^{(i)}\} \leftarrow$ \Optimization{$P.solution$}\label{alg:overview:dbcbs:opt}\marklineOne{-25pt} \\
                \If{$\{\seqX^{(i)} \}, \{\seqU^{(i)}\}$ successfully computed}{
                    \Report{$\{\seqX^{(i)} \}, \{\seqU^{(i)}\}$ } \Comment*{New solution found}
                }
            }
            {    
            $\langle i, j, \vx^{(i)}, \vx^{(j)}, k \rangle \leftarrow  \GetConstraints(C)$ \Comment*{Get constraints from conflicts}
            \ForEach{$c \:\in\: \{ i,j \}$}{ \label{alg::dbcbs:constraint_loop_begin}
                $P' \leftarrow \NewNode(solution: P.solution, constraints: P.constraints \cup \{\langle c, \vx^{(c)}, k \rangle\}) $ \\
                $P'.solution[c] \leftarrow$ \DiscontinuityBoundedAstar{$\vx_s^{(c)}, \vx_f^{(c)}, \sW_{\mathrm{free}}, \sM^{(c)}, 
                \delta, P'.constraints$}  \label{alg:dbcbs:dbAstar_with_constr} \;
                $P'.cost = \GetSolutionCost(P'.solution)$ \\
                $\PriorityQueueInsert(\mathcal O, P')$ \label{alg::dbcbs:constraint_loop_end}
            }
            }
        }           
    }
\end{algorithm}

\begin{algorithm}[t]
    \caption{db-A* with Dynamic Obstacles}
    \label{alg:dbAstarMRS}
    \DontPrintSemicolon
    \SetVlineSkip{2pt}

    \SetKwFunction{NearestNeighborInit}{NearestNeighborInit}
    \SetKwFunction{NearestNeighborQuery}{NearestNeighborQuery}
    \SetKwFunction{NearestNeighborAdd}{NearestNeighborAdd}
    \SetKwFunction{PriorityQueuePop}{PriorityQueuePop}
    \SetKwFunction{PriorityQueueInsert}{PriorityQueueInsert}
    \SetKwFunction{PriorityQueueUpdate}{PriorityQueueUpdate}
    \SetKwFunction{CheckForCollision}{CheckForCollision}
    \SetKwFunction{Node}{Node}

     \tcc{Main changes compared to db-A* are highlighted.}
    \KwData{$\vx_s, \vx_f, \sW_{\mathrm{free}}, \sM, \delta, \sC$}
    \KwResult{$\seqX, \seqU, K$ or Infeasible}
    
    $\mathcal T_m \leftarrow \NearestNeighborInit(\mathcal M)$ \Comment*{Use start states of motions (excl. position)} \label{alg:dbAstarMRS:Tm}
    $\mathcal T_n \leftarrow \NearestNeighborInit(\{\vx_s\})$ \Comment*{capture explored vertices (incl. position)}\label{alg:dbAstarMRS:Tn}
    \tikzmk{A}$\mathcal O \leftarrow \{\Node(\vx: \vx_s, g: 0, h: h(\vx_s), A: \{(g: 0, p: None, a: None)\}) \}$ \label{alg:dbAstarMRS:Oinit} \Comment*{Initialize open priority queue} \marklineTwo{-10pt}
    \While{$|\mathcal O| > 0$}{
        $n \leftarrow \PriorityQueuePop(\mathcal O)$ \Comment*{Lowest f-value} \label{alg:dbAstarMRS:Opop}
        \If{$d(n.\vx, \vx_f) \leq \delta$ \label{alg:dbAstarMRS:sol_cond}}{
            \Return $\seqX, \seqU, K$ \Comment*{Traceback solution} \label{alg:dbAstarMRS:sol}
        }
        \Comment{Find applicable motion primitives with discontinuity up to $\alpha \delta$}
        $\mathcal M' \leftarrow \NearestNeighborQuery(\mathcal T_m, n.\vx, \alpha\delta)$\label{alg:dbAstarMRS:Mprime}\;
        \ForEach{$m\in \mathcal M'$ \label{alg:dbAstarMRS:expand1}}{
        \If{$n.\vx\oplus m \notin \sW_{\mathrm{free}}$ \label{alg:dbAstarMRS:collision}}{
            \Continue \label{alg:dbAstarMRS:expand2} \Comment*{entire motion is not collision-free}
        }
        $g_t \leftarrow n.g + cost(m)$ \Comment*{tentative g score for this action}
        \tikzmk{A}\ForEach{$c\in \sC$ \label{alg:dbAstarMRS:constraint_loop_begin}}{
            $\vx' = m[\lfloor \frac{c.k - n.g}{\Delta t} \rfloor]$ \Comment*{Motion primitive state for checking} \label{alg::dbAstar:get_c.x}\;
            \If{$d(\vx',c.\vx) \leq \delta$} { \label{alg:dbAstarMRS:distance}
                \Continue loop \cref{alg:dbAstarMRS:expand1}  \label{alg:dbAstarMRS:check_c.x_distance} \Comment*{$m$ does not obey all constraints}
            }\label{alg:dbAstarMRS:check_collision}
        }\marklineSix{-20pt}
        \Comment{find already explored nodes within $(1-\alpha)\delta$}
        $\mathcal N' \leftarrow \NearestNeighborQuery(\mathcal T_n, n.\vx\oplus m, (1-\alpha)\delta)$ \label{alg:dbAstarMRS:Nprime}\;
        \eIf{$\mathcal N' = \emptyset$}{
            \tikzmk{A}$\PriorityQueueInsert(\mathcal O, \Node(\vx: n.\vx \oplus m, g: g_t, h: h(n.\vx\oplus m), A: \{(g: g_t, p: n, a: m)\})$\marklineThree{-25pt} \label{alg:dbAstarMRS:Oadd}\;
            $\NearestNeighborAdd(\mathcal T_n, n.\vx \oplus m)$ \label{alg:dbAstarMRS:TnGrow}\;
        }{
            \ForEach{$n'\in \mathcal N'$}{
                \If{$g_t < n'.g$ }{
                \Comment*{This motion is better than a known motion} \label{alg:dbAstarMRS:Oupdate1}
                    $n'.g = g_t$ \Comment*{Update cost}
                    \tikzmk{A}$n'.A = n'.A \cup \{(g: g_t, p: n, a: m)\}$ \Comment*{Add parent and action}\label{alg:dbAstarMRS:newRewiring}\marklineOne{-35pt}
                    $\PriorityQueueUpdate(\mathcal O, n')$ \label{alg:dbAstarMRS:Oupdate2}
                }
            }
        }
        }
    }
    \Return Infeasible
\end{algorithm}

\subsection{Properties and Extensions}

CBS is a complete and optimal algorithm for MAPF, i.e., if a solution exists it will find it and the first reported solution has the lowest possible cost (sum of costs of all agents).
The proof~\cite{cbs} uses two arguments: i) for completeness, it is shown that the search will visit all states that can contain the solution, i.e., no potential solution paths are pruned, and ii) for optimality, it is shown that CBS visits the solutions in increasing order of costs ($f$-value) and can therefore not have missed a lower-cost solution once it terminates.

The continuous time and continuous space renders the full enumeration proof argument infeasible for MRMP problems.
Instead, we consider probabilistic completeness (PC; the probability of finding a solution if one exists is 1 in the limit over search effort) and asymptotic optimality (AO; the cost difference between the reported solution and an optimal solution approaches zero in the limit over search effort).
We assume that there exists a finite $\delta > 0$ such that the trajectory with discontinuity is optimized correctly using the trajectory optimization. 
In the following, we show that db-CBS is both probabilistically complete and asymptotically optimal and therefore an anytime planner, similar to sampling-based methods such as SST*.

\begin{theorem}
    \label{theorem:ao}
    The db-CBS motion planner in \cref{alg:dbcbs} is asymptotically optimal, i.e.
    \begin{equation}
        \lim_{n\to\infty} P(\{ c_n - c^* > \epsilon \}) = 0, \; \forall \epsilon > 0,
    \end{equation}
    where $c_n$ is the cost in iteration $n$ and $c^*$ is the optimal cost.
\end{theorem}
\begin{proofsketch}
    Each iteration in \cref{alg:dbCBS:expand2} operates on a discrete search problem over the graph implicitly defined by the finite set of motion primitives. 
    We find an optimal solution, if one exists, within this discretization by \cite[Theorem 1]{cbs}.
    Subsequent iterations add more motion primitives, reduce $\delta$, and therefore produce a larger discrete search graph, for which the optimality proof of CBS still holds.
    In the limit ($|\sM|\to\infty$, $\delta\to 0$), db-CBS will report the optimal solution.

    This argument is the same as the more detailed version in \cite[Theorem 1 and 2]{solis2021}, for combining PRM and CBS. A formal version of this proof, including bounding the actual probability, is given in \cite[Theorem 2]{dbastar}.
\end{proofsketch}

\begin{theorem}
    \label{theorem:pc}
    The db-CBS motion planner in \cref{alg:dbcbs} is probabilistically complete.
\end{theorem}
\begin{proof}
    This directly follows from \cref{theorem:ao}, since asymptotic optimality implies probabilistic completeness.
\end{proof}

This result compares to the baseline algorithms as follows.
SST* is probabilistically complete and asymptotically near-optimal, where the near-optimality stems from a fixed hyperparameter similar to the time-varying $\delta$ in db-CBS.
K-CBS is probabilistically complete and could achieve asymptotic optimality by using a framework such as AO-x~\cite{AO-RRT}.
S2M2 is neither probabilistically complete nor asymptotically optimal, because it uses pessimistic approximations and priority-based search.
We show concrete counter-examples in which S2M2 fails to find a solution in \cref{sec:results}.

We note that db-CBS can naturally plan for heterogeneous teams of robots, where individual team members may have unique dynamics or collision shapes.
This does not require any algorithmic changes, as the low-level planner can use different motion primitives and constraints are defined for each robot themselves, and conflicts can be detected using arbitrary collision shapes.
On the optimization side, we rely on nonlinear optimization and stacked states  which naturally extends to varied state and action dimensions.

\section{Results}
\label{sec:results}
We compare our method with other multi-robot kinodynamic motion planners on the same problem instances. 
We consider robots with three to five-dimensional states: unicycle, $2^{\text{nd}}$ order unicycle, double integrator, and car with trailer, see~\cite{dbastar} for dynamics and bounds.
For the double integrator, we use $v \in [-0.5, 0.5]~\si{m/s}$, $a \in [-2.0, 2.0]~\si{m/s^2}$.
For testing scenarios we focus on cases where the workspace has small dimensions, thus it is challenging for robots to find collision-free motions in a time-optimal way.

In each environment, we compare our method with K-CBS~\cite{kcbs}, S2M2~\cite{s2m2}, and joint space SST*~\cite{sst}.
We analyze success rate $(p)$, computational time until the first solution is found $(t)$, and cost of the first solution $(J)$. An instance is not solved successfully if an incorrect solution is returned, or no solution is found after \SI{5}{min}.
To compare relative improvements over baselines, we use a notion of regret, i.e., $r^{\text{s2m2}} = (J^{\text{s2m2}} - J^{\text{dbcbs}}) / J^{\text{s2m2}}$.

We use OMPL~\cite{OMPL} and FCL~\cite{FCL} for K-CBS, SST* and db-CBS.
For fair comparison, we set the low-level planner of K-CBS to SST*, and modify the collision checking in order to support robots of any shape. For S2M2, we use the Euler integration to propagate the robot state and bound control actions.
For K-CBS and S2M2 we use the respective publicly available implementations from the authors.
For K-CBS, we use a merge bound of $\infty$, effectively disabling the merging.
When comparing with S2M2, we use a spherical first order unicycle model and quadrupled angular velocity bounds $\omega \in [-2, 2]~\si{rad/s}$, since the provided code was unable to solve most of our very dense planning problems otherwise.
For \textbf{db-CBS} \cref{alg:dbcbs} and \cref{alg:dbAstarMRS} are implemented in C++, and the benchmarking script is written in Python. 
For optimization, we extend Dynoplan~\cite{dynoplan}, which uses the DDP solver Crocoddyl~\cite{mastalli20crocoddyl}.
We use a workstation (AMD Ryzen Threadripper PRO 5975WX @ 3.6 GHz, 64 GB RAM, Ubuntu 22.04).
Example results for all algorithms and real-world experiments are available in the supplemental video. The code and benchmark problems are publicly available. 

\subsection{Canonical Examples}

\begin{figure}
	\setlength{\tabcolsep}{0.0em} %
	\centering
	\begin{tabular}{ccc}
		\includegraphics[width=0.32\linewidth]{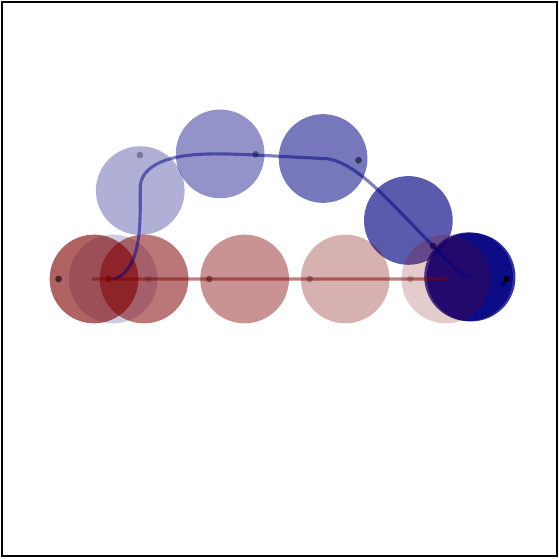}      &
        \includegraphics[width=0.32\linewidth]{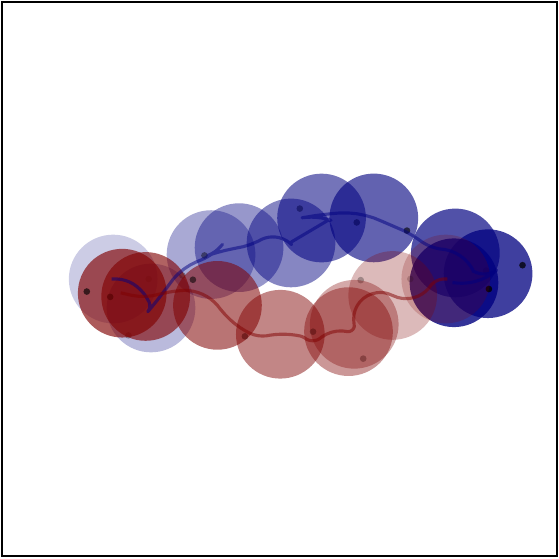} &
		\includegraphics[width=0.32\linewidth]{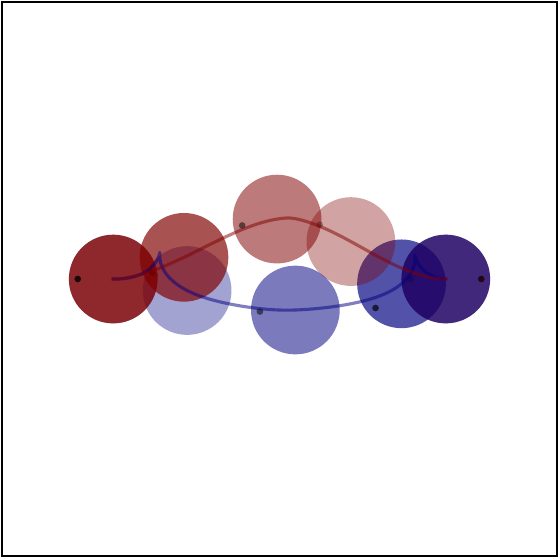}   \\
        \includegraphics[width=0.32\linewidth]{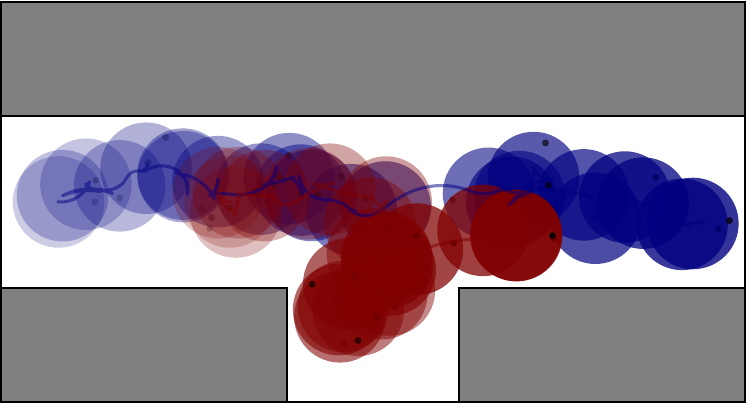} &
		\includegraphics[width=0.32\linewidth]{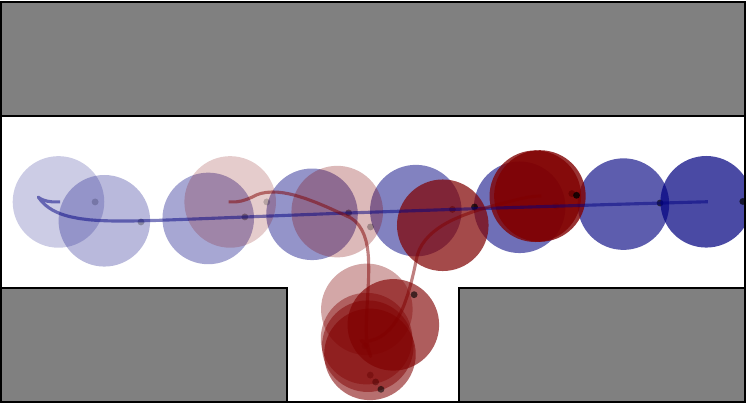}      &
		\includegraphics[width=0.32\linewidth]{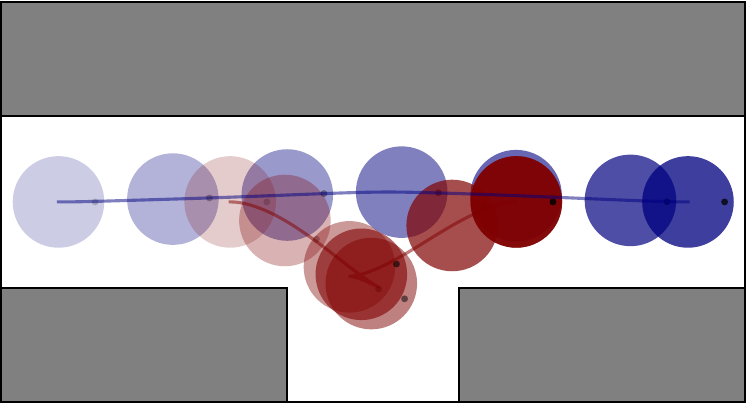}   \\
        \includegraphics[width=0.32\linewidth]{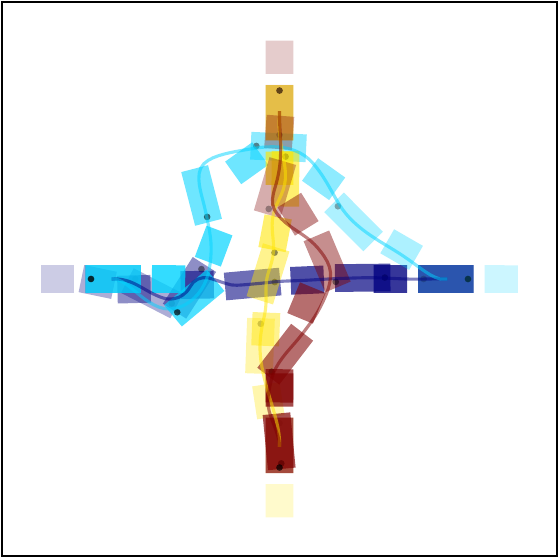} &
		\includegraphics[width=0.32\linewidth]{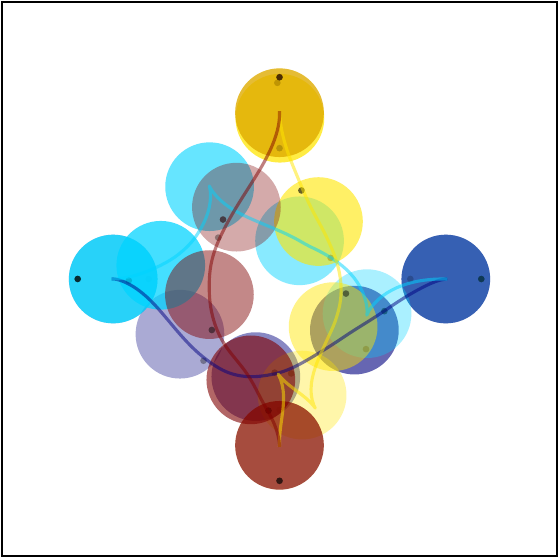}  &
        \includegraphics[width=0.32\linewidth]{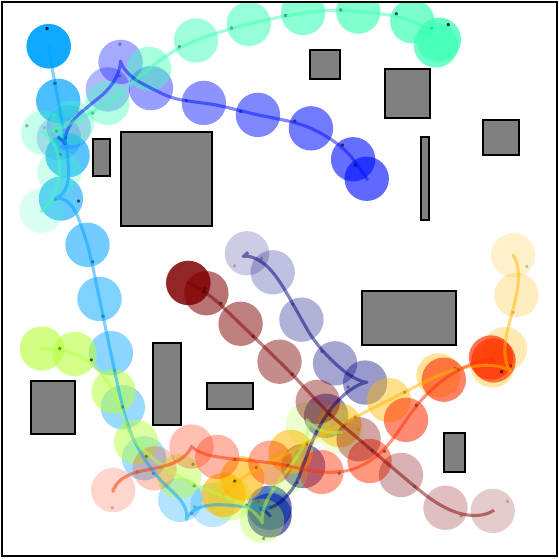}      \\
	\end{tabular}
	\caption{Methods are listed from left to right. Upper row: S2M2, K-CBS, db-CBS with \emph{swap}. Middle row: SST*, S2M2, db-CBS with \emph{alcove}. 
    Last row shows instances where only db-CBS is able to return a solution: 4 cars with trailer with \emph{swap}, 4 unicycles with \emph{swap}, 8 unicycles with \emph{random} start and goal states.}
	\label{fig:result}
\end{figure}

We compare the algorithms on canonical examples often seen in the literature, see rows 1--3 in \cref{tab:unified} and \cref{fig:result}.

\subsubsection{Swap} the environment has no obstacles and 2 robots have to swap their positions, i.e, each robot's start position is the goal position of the other robot. In this environment all baselines succeed with SST* and S2M2 being the fastest to find a solution. However, their solution costs are higher compared to db-CBS, even though less than of K-CBS.

\begin{table*}
\caption{Canonical examples, scalability with varying robot numbers, and heterogeneous systems. \\
Median values over 10 trials per row. Bold entries are the best for the row, \textemdash no solution found, $\star$ not tested.}
\centering
\setlength{\tabcolsep}{5pt}
\begin{tabular}{c || c || r|r|r|r || r|r|r|r || r|r|r|r || r|r|r|r}
\# & Instance & \multicolumn{4}{c||}{SST*} & \multicolumn{4}{c||}{S2M2} & \multicolumn{4}{c||}{K-CBS} & \multicolumn{4}{c}{db-CBS}\\&  & $p$ & $t [s]$ & $J [s]$ & $r [\%]$ & $p$ & $t [s]$ & $J [s]$ & $r [\%]$ & $p$ & $t [s]$ & $J [s]$ & $r [\%]$ & $p$ & $t [s]$ & $J [s]$ & $r [\%]$\\\hline\hline
1 & swap & 0.3 & \bfseries 0.1 & 29.8 & 49 & \bfseries 1.0 & \bfseries 0.1 & 16.0 & 5 & \bfseries 1.0 & 21.5 & 30.4 & 50 & \bfseries 1.0 & 3.6 & \bfseries 15.2 & \bfseries 0\\\hline
2 & alcove & 0.9 & 4.2 & 45.5 & 50 & \bfseries 1.0 & \bfseries 0.5 & 29.2 & 22 & 0.0 & \textemdash & \textemdash & \textemdash & \bfseries 1.0 & 44.0 & \bfseries 22.8 & \bfseries 0\\\hline
3 & at goal & 0.5 & \bfseries 0.5 & 33.8 & 47 & 0.0 & \textemdash & \textemdash & \textemdash & 0.0 & \textemdash & \textemdash & \textemdash & \bfseries 1.0 & 58.9 & \bfseries 17.6 & \bfseries 0\\\hline
\hline
4 & rand (N=2) & 0.1 & 0.3 & \bfseries 18.7 & 58 & 0.7 & \bfseries 0.1 & 27.5 & 8 & \bfseries 1.0 & 2.0 & 65.0 & 63 & \bfseries 1.0 & 6.0 & 25.6 & \bfseries 0\\\hline
5 & rand (N=4) & 0.0 & \textemdash & \textemdash & \textemdash & 0.3 & \bfseries 0.4 & \bfseries 52.0 & 13 & 0.6 & 61.9 & 144.4 & 59 & \bfseries 1.0 & 28.7 & 55.1 & \bfseries 0\\\hline
6 & rand (N=8) & 0.0 & \textemdash & \textemdash & \textemdash & 0.0 & \textemdash & \textemdash & \textemdash & 0.0 & \textemdash & \textemdash & \textemdash & \bfseries 0.1 & \bfseries 256.6 & \bfseries 146.9 & \bfseries 0\\\hline
\hline
7 & rand hetero (N=2) & 0.1 & 157.2 & 33.6 & 29 & $\star$ & $\star$ & $\star$ & $\star$ & 0.5 & 17.2 & 50.9 & 62 & \bfseries 1.0 & \bfseries 2.2 & \bfseries 14.2 & \bfseries 0\\\hline
8 & rand hetero (N=4) & 0.0 & \textemdash & \textemdash & \textemdash & $\star$ & $\star$ & $\star$ & $\star$ & 0.2 & 145.6 & 89.9 & 65 & \bfseries 0.8 & \bfseries 7.5 & \bfseries 37.1 & \bfseries 0\\\hline
9 & rand hetero (N=8) & 0.0 & \textemdash & \textemdash & \textemdash & $\star$ & $\star$ & $\star$ & $\star$ & 0.0 & \textemdash & \textemdash & \textemdash & \bfseries 0.3 & \bfseries 271.8 & \bfseries 86.4 & \bfseries 0\\
\end{tabular}
\label{tab:unified}
\end{table*}    

\begin{table}
\caption{Runtime of SST* ($\star$), K-CBS ($\dagger$), and db-CBS($\ddagger$) in seconds on the swap problem (median over 10 trials).}
\centering
\setlength{\tabcolsep}{3pt}
\begin{tabular}{c  || r|r|r || r|r|r || r|r|r || r|r|r}
N  & \multicolumn{3}{c||}{unicycle $1^{\mathrm{st}}$} & \multicolumn{3}{c||}{double int.} & \multicolumn{3}{c||}{car with trailer} & \multicolumn{3}{c}{unicycle $2^{\mathrm{nd}}$}\\& $\star$& $\dagger$& $\ddagger$& $\star$& $\dagger$& $\ddagger$& $\star$& $\dagger$& $\ddagger$& $\star$& $\dagger$& $\ddagger$\\\hline\hline1 & \bfseries 0.2 & 1.0 & 1.2 & \bfseries 0.0 & 1.0 & 0.2 & \bfseries 0.3 & 1.0 & 2.5 & \bfseries 1.5 & 14.6 & 7.0\\\hline2 & 2.4 & \bfseries 2.0 & 2.1 & \textemdash & 2.0 & \bfseries 0.4 & \textemdash & 10.1 & \bfseries 5.3 & \textemdash & 39.3 & \bfseries 11.3\\\hline3 & \textemdash & 34.6 & \bfseries 4.1 & \textemdash & \textemdash & \bfseries 0.5 & \textemdash & \textemdash & \bfseries 11.9 & \textemdash & \textemdash & \bfseries 20.9\\\hline4 & \textemdash & 214.1 & \bfseries 8.0 & \textemdash & \textemdash & \bfseries 1.1 & \textemdash & \textemdash & \bfseries 86.0 & \textemdash & \textemdash & \bfseries 23.5\\
\end{tabular}
\label{tab:swap}
\end{table}

\subsubsection{Alcove}\label{res:alcove} the environment requires one robot to move into an alcove temporarily in order to let the second robot pass. All baselines and db-CBS are able to solve this problem, except K-CBS. 

\subsubsection{At Goal} similar to \textit{Alcove}, but one robot is already at its goal state and needs to move away temporarily to let the second robot pass. Just SST* and db-CBS are able to find a solution, but with high difference between their solution cost. The failure of S2M2 is due to collision violation.

Both \emph{alcove} and \emph{at goal} should be solvable by K-CBS for finite merge bounds (a feature that was not readily available in the provided code). However, in this case the algorithm essentially becomes joint space SST* and would produce solutions with a high cost as seen in \cref{tab:unified}.

\subsection{Scalability} 

\subsubsection{Robot Number} we generate obstacles, start and goal states for up to 8 robots, see rows 4--6 in \cref{tab:unified}. S2M2 and K-CBS are able to solve this problem with up to 4 robots, however inconsistently. In addition both of them are failing to return a solution for 8 robot cases, while db-CBS scales up successfully.

\subsubsection{Heterogeneous Systems} we generate obstacles, start and goal states for a team of up to 8 heterogeneous robots, see rows 7--9 in \cref{tab:unified}.
S2M2 is not tested, since it does not support all robot dynamics considered here. 
Even though SST* solves instances with 2 robots with significantly better cost compared to K-CBS, it struggles to scale to more robots. K-CBS is able to handle up-to 4 heterogeneous robots, but the solution cost is very high. The success rate and solution cost are the best for db-CBS across all settings.

Note that in these random instances the shown median of $J$ is not directly comparable unless the success rate $p=1$.
In the table the positive regret $r$ for all baselines shows that db-CBS consistently computes lower-cost solutions.

\subsubsection{State Dimension}\label{res:dimension} in \cref{tab:swap} we report the solution quality of \textit{swap} problem instances with different dynamics and varying robot numbers (1--4), while keeping the start and goal states unchanged. None of the methods, except db-CBS, succeed to solve the problem with more than two robots across different dynamics. However, K-CBS is able to solve up to 4 unicycles ($1^{st}$ order), while the success rate of SST* is very low if the robot number is greater than 1.

\subsection{Real-Robot Demo}
The real-world experiments are conducted inside a $3.5 \times 3.5 \times 2.75~\si{m^3}$ room. We use Bitcraze Crazyflie 2.1 drones and control them using Crazyswarm2. 
We consider two scenarios with 4 robots modeled with 2D double integrator dynamics (and thus do not include S2M2 in the discussion). 
We first test the \textit{swap} example.
K-CBS and SST* can not solve this problem similar to results observed in \cref{res:dimension}, while db-CBS takes $\SI{12.1}{s}$ to return a trajectory with $\SI{11.4}{s}$ cost.
The second scenario is an environment including a wall with a small window and 2 robots on each side of the wall. The goal states require the robots to pass through the small window, see \cref{fig:main}. Only db-CBS succeeds finding a solution within $\SI{80}{s}$ with $\SI{21.1}{s}$ cost.

\section{Conclusion}

In this paper, we present db-CBS, an efficient, probabilistically complete and asymptotically optimal motion planner for a heterogeneous team of robots that considers robot dynamics and control bounds. Db-CBS solves the multi-robot kinodynamic motion planning problem by finding collision-free trajectories with a bounded discontinuity and optimizes them in joint configuration space. Empirically, we show that db-CBS finds solutions with a significantly higher success rate and better solution cost compared to the existing state-of-the-art. Finally, we validate our planner on two real-world challenging problem instances on flying robots.

In the future, we are interested in solving problems with more robots and over longer time horizons, for example by combining db-A* with a more advanced version of CBS, such as Enhanced CBS (ECBS), and by repairing discontinuities locally.

\balance
\printbibliography

\end{document}